\documentclass[sigconf,letterpaper,9pt]{acmart}
\acmConference[FAT*'19]{ACM Conference on Fairness, Accountability, and
Transparency}{January 2019}{Atlanta, Georgia USA}
\acmYear{2019}
\copyrightyear{2019}

\usepackage[utf8]{inputenc} 
\usepackage[T1]{fontenc}    
\usepackage{hyperref}       
\usepackage{url}            
\usepackage{booktabs}       
\usepackage{amsfonts}       
\usepackage{nicefrac}       
\usepackage{microtype}      
\usepackage{amsmath,amssymb,amsfonts,amsthm}
\usepackage{natbib}
\usepackage{graphicx}
\usepackage{enumitem}
\usepackage{mathtools}
\usepackage{float}
\usepackage{appendix}

\begingroup
    \makeatletter
    \@for\theoremstyle:=definition,remark,plain\do{%
        \expandafter\g@addto@macro\csname th@\theoremstyle\endcsname{%
            \addtolength\thm@preskip\parskip
            }%
        }
\endgroup

\makeatletter
\newtheorem*{rep@theorem}{\rep@title}
\newcommand{\newreptheorem}[2]{%
\newenvironment{rep#1}[1]{%
 \def\rep@title{#2 \ref{##1}}%
 \begin{rep@theorem}}%
 {\end{rep@theorem}}}
\makeatother

\DeclareMathOperator*{\argmax}{arg\,max}

\newtheorem{theorem}{Theorem}[section]
\newtheorem{lemma}{Lemma}[section]

\newtheorem{corollary}{Corollary}[section]
\newreptheorem{theorem}{Theorem}
\newreptheorem{lemma}{Lemma}

\theoremstyle{definition}
\newtheorem{definition}{Definition}[section]
\newtheorem{example}{Example}[section]

\newcommand{\X}[0]{\mathcal{X}}
\newcommand{\Y}[0]{\mathcal{Y}}
\newcommand{\R}[0]{\mathbb{R}}

\newcommand{\iutil}[0]{\mathcal{U}_\Delta}
\newcommand{\iutiln}[0]{\mathcal{U}}
\newcommand{\sutil}[0]{\mathcal{S}_{+}}

\newcommand{\br}[0]{\Delta}
\newcommand{\x}[0]{x}
\newcommand{\ex}[0]{\mathbb{E}}
\newcommand{\ind}[1]{\mathbb{I}\{#1\}}
\newcommand{\prob}[1]{\mathbb{P}(#1)}

\newcommand{\accn}[0]{\mathcal{A}}
\newcommand{\acc}[0]{\mathcal{A}_{\Delta}}

\newcommand{\sgap}[0]{\mathcal{G}}
\newcommand{\acost}[0]{c_a}
\newcommand{\bcost}[0]{c_b}
\newcommand{\scost}[0]{\mathcal{B}_{+}}
\newcommand{\scostg}[1]{\mathcal{B}_{+, #1}}
\newcommand{\ol}[0]{\ell}
\newcommand{\old}[1]{\ell_{\Delta_#1}}
\newcommand{\olcost}[0]{c_L}
\newcommand{\olrv}[1]{L_{+, #1}}
\newcommand{\olcdf}[1]{F_{+, #1}}
\newcommand{\leibd}[0]{\mathop{}\!\mathrm{d}}
\newcommand{\accs}[1]{\mathcal{A}_{#1}}
\newcommand{\opt}[0]{^{*}}

\RequirePackage{color}
\newcommand{\todo}[1]{}
\renewcommand{\todo}[1]{{\color{red} J: {#1}}}

\begin{document}
	\title{The Social Cost of Strategic Classification}
	\author{Smitha Milli}
	\email{smilli@berkeley.edu}
	\affiliation{\institution{University of California, Berkeley}}
	\author{John Miller}
	\email{miller\_john@berkeley.edu}
	\affiliation{\institution{University of California, Berkeley}}
	\author{Anca D. Dragan}
	\email{anca@berkeley.edu}
	\affiliation{\institution{University of California, Berkeley}}
	\author{Moritz Hardt}
	\email{hardt@berkeley.edu}
	\affiliation{\institution{University of California, Berkeley}}
	\begin{abstract}
Consequential decision-making typically incentivizes individuals to behave
strategically, tailoring their behavior to the specifics of the decision rule.
A long line of work has therefore sought to counteract strategic behavior by
designing more conservative decision boundaries in an effort to increase
robustness to the effects of strategic covariate shift.

We show that these efforts benefit the institutional decision maker at the
expense of the individuals being classified. Introducing a notion of social
burden, we prove that any increase in institutional utility necessarily leads to
a corresponding increase in social burden. Moreover, we show that the negative
externalities of strategic classification can disproportionately harm
disadvantaged groups in the population.

Our results highlight that strategy-robustness must be weighed against
considerations of social welfare and fairness.
\end{abstract}

	 
	 \copyrightyear{2019} 
\acmYear{2019} 
\setcopyright{acmcopyright}
\acmConference[FAT* '19]{FAT* '19: Conference on Fairness, Accountability, and
Transparency}{January 29--31, 2019}{Atlanta, GA, USA}
\acmBooktitle{FAT* '19: Conference on Fairness, Accountability, and
Transparency (FAT* '19), January 29--31, 2019, Atlanta, GA, USA}
\acmPrice{15.00}
\acmDOI{10.1145/3287560.3287576}
\acmISBN{978-1-4503-6125-5/19/01}

	
	\maketitle

	\section{Introduction}
\label{sec:intro}


As machine learning increasingly supports consequential decision making, its
vulnerability to manipulation and gaming is of growing concern. When individuals
learn to adapt their behavior to the specifics of a statistical decision rule,
its original predictive power will deteriorate. This widely observed
empirical phenomenon, known as Campbell's Law or Goodhart's Law, is often
summarized as: ``Once a measure becomes a target, it ceases to be a good
measure''~\cite{strathern1997improving}.

Institutions using machine learning to make high-stakes decisions naturally wish
to make their classifiers robust to strategic behavior.  A growing line of work
has sought algorithms that achieve higher utility for the institution in
settings where we anticipate a strategic response from the the classified
individuals~\citep{dalvi2004adversarial,bruckner2012static,hardt2016strategic}.
Broadly speaking, the resulting solution concepts correspond to more
conservative decision boundaries that increase robustness to some form of
distributional shift.

But there is a flip side to strategic classification. As insitutional utility
increases as a result of more cautious decision rules, honest individuals worthy
of a positive classification outcome may face a higher bar for success.

The costs incurred by individuals as a consequence of strategic classification
are by no means hypothetical, as the example of lending shows. In the United
States, credit scores are widely deployed to allocate credit. However, even
creditworthy individuals routinely engage in artificial practices intended to
improve their credit scores, such as opening up a certain number of credit
lines in a certain time period~\cite{citron2014scored}.
%
%

In this work, we study the tension between accuracy to the institution and 
impact to the individuals being classified. We first introduce a general measure of the cost of strategic classification,
which we call the \emph{social burden}. Informally, the social burden measures
the expected cost that a positive individual needs to incur to be correctly
classified correctly.

For a broad class of cost functions, we prove there exists an intrinsic
trade-off between institutional accuracy and social burden: any increase in
institutional accuracy comes at an increase in social burden.  Moreover, we
precisely characterize this trade-off and show the commonly considered
Stackelberg equilibrium solution achieves maximal institutional accuracy
at the expense of maximal social burden.

Equipped with this generic trade-off result, we turn towards a more careful
study of how the social burden of strategic classification impacts different
subpopulations. We find that the social burden can fall disproportionally on
disadvantaged subpopulations, under two different notions by which one group can
be \emph{disadvantaged} relative to another group.
%
%
Furthermore, we show that as the institution improves its accuracy, it
exacerbates the gap between the burden to an advantaged and disadvantaged group.
Finally, we illustrate these conditions and their consequences with a case study
on FICO data.
%

\subsection{Our Contributions}
In this paper, we make the following contributions:
\begin{enumerate}
\item We prove a general result demonstrating the trade-off between
institutional accuracy and individual utility in the strategic setting. Our
theoretical characterization is supplemented with examples illustrating when an
institution would prefer to operate along different points in this trade-off
curve.
\item We show fairness considerations inevitably arise in the strategic setting.
When individuals incur cost as a consequence of making a classifier robust to
strategic behavior, we show the costs can disproportionally fall by
disadvantaged subpopulations. Furthermore, as the institution increases its
robustness, it also increases the
disparity between subpopulations.
\item Using FICO credit data as a case-study, we empirically validate our
modeling assumptions and illustrate both the general trade-offs and fairness
concerns involved with strategic classification in a concrete setting.
\end{enumerate}
Reflecting on our results, we argue that the existing view of strategic
classification has been \emph{instituition-centric}, ignoring the social burden
resulting from improved institutional utility. Our framework makes it possible
to select context-specific trade-offs between institutional and individual
utility, leading to a richer space of solutions.

Another key insight is that discussions of strategy-robustness must go hand in
hand with considerations of fairness and the real possibility that
robustness-promoting mechanisms can have disparate impact in different segments
of the population.

	\section{Model}\label{sec:model}
\paragraph{\textbf{Strategic classification.}}
Throughout this work, we consider the binary classification setting.  Each
individual has features $x \in \X$ and a label $y \in \Y = \{0, 1\}$.  The
institution publishes a classifier $f : \X \rightarrow \Y$.  In the \emph{non-strategic setting}, the institution maximizes the \emph{non-strategic utility}, which is simply the
classification accuracy of $f$:
\begin{align*}
    & \iutiln(f) = \prob{f(x) = y} \,.
\end{align*}

In the \emph{strategic setting}, the individual can modify their
features, and the institution aims to preempt the individual's strategic
manipulation. In response to the institution's classifier $f$, the individual can
change her features $x$ to new features $x'$.  However, modification incurs a
cost given by $c : \X \times \X \rightarrow \R_{\geq 0}$.  The individual then
receives an \emph{individual utility} $u_x(f, x') = f(x')- c(x, x')$, which
trades off between the cost of manipulation $c(x, x')$ and the benefits of
classification $f(x')$.

The institution models the individual $x$ as maximizing their utility $u_x(f,
x')$ and acting according to  the \emph{best-response} to the classifier $f$:
\begin{align*}
    \Delta(\x; f) = \argmax_{\x'} u_x(f, x').
\end{align*}
When it is clear from context we will drop the dependence on $f$ and write the
individual's best response as $\br(x)$. Although $\br(x)$ may not have a unique
maximizer, it is assumed that the individual $x$ does not adapt her features if
she is already accepted by the classifier,~i.e.~$f(x) = 1$, or if there is no
maximizer $x'$ she can move to such that $f(x') = 1$. In cases where the
individual does adapt, let $x'$ be an arbitrary maximizer such that $f(x') = 1$.
In practice, it is unlikely individuals actually play best-response solutions,
and we will discuss as appropriate the impact of deviations from best-response
play.
%


Given this model, the institution aims to maximize the 
\emph{strategic
utility}, which
measures accuracy \emph{after} individual responses:
\begin{align*}
    \iutil(f) = \prob{f(\br(\x)) = y}.
\end{align*}

For example, imagine that the institution is trying to rank pages on a social
network. Although the number of likes a page has may be predictive, it is also
an easy feature to game.  Therefore, models with high strategic utility will
assign low weight to this feature, even if it is useful in the static setting. Henceforth, we will refer to the strategic utility as simply the \emph{institutional utility}.

\paragraph{\textbf{Social burden.}} Focusing purely on maximizing $\iutil$, as
done in prior work, ignores the cost a classifier imposes on individuals 
\cite{bruckner2011stackelberg, hardt2016equality,
dong2018strategic}. To account for these costs, we define the \emph{individual
burden} of a classifier $f$ as the minimum cost an individual needs to incur in
order to be classified positively: $b_f(x) = \min_{f(x') = 1} c(x, x')$.

For positive individuals with $y=1$, a high individual burden means the individual
has to incur great cost to obtain the correct classification.  To quantity this
cost, we introduce the \emph{social burden}, defined as the expected individual
burden of positive individuals.
\begin{definition}[Social burden] 
    The social burden of a classifier $f$ is defined as $\scost(f) = \ex \left
[b_f(x) \mid y = 1 \right]$.
\end{definition}

The social burden measures the expected cost that positive individuals would
need
to incur to be classified positively, regardless of whether the best response
$\Delta(x)$
indicates that they should adapt.
One
could
imagine
other ways of
measuring
the impact on individuals, such as the expected utility of positive
individuals,
$\ex
\left
[
u_x(f, \Delta(x)) \mid y = 1
\right ]$, or the corresponding measure over all individuals,
rather
than
only
positive
individuals.
Most
of
our
results
still
hold
for
these alternative
measures, and we relegate discussion about the choice of social burden to
Section
\ref{sec:disc}.



\paragraph{\textbf{Assumptions on cost function.}} 
While there are many possible models for the cost function, we restrict our
attention to a natural set of cost functions that we call \textit{outcome
monotonic}. Outcome monotonic costs capture two intuitive properties: (1)
Monotonically improving one's outcome requires monotonically increasing amounts
of work, and (2) it is zero cost to worsen one's outcome. This captures the
intuition that, for example, it is harder to pay back loans than it is to go
bankrupt.

\begin{definition}[Outcome likelihood]
	The outcome likelihood of an individual $x$ is $\ol(x) = \prob{Y = 1 \mid X =
x}$.
\end{definition}
We assume that all individuals have a positive outcome likelihood, i.e, $\ol(x)
> 0$ for all $x$.

\begin{definition}[Outcome Monotonic Cost]
\label{def:cost}
A cost function $c : \X \times \X \rightarrow \R_{\geq 0}$ is outcome monotonic
if for any $x, x', x^{*} \in X$ the following properties hold.
\begin{itemize}
    \item \textit{Zero-cost to move to lower outcome likelihoods.} $c(x, x') >
0$ if and only if $\ol(x') > \ol(x)$.
\item \textit{Monotonicity in first argument.} $c(x, x^{*}) > c(x', x^{*}) > 0$
if and only if $\ell(x^*) > \ell(x') > \ell(x)$.
    \item \textit{Monotonicity in second argument.} $c(x, x^{*}) > c(x, x') >
0$ if and only if $\ol(x^*) > \ol(x') > \ol(x)$.
\end{itemize}
\end{definition}

Under these assumptions, we can equivalently express the cost as a cost over
outcome likelihoods, $\olcost : \ell(\X) \times \ell(\X) \rightarrow \R_{\geq
0}$, defined in the following lemma.
\begin{lemma}
    \label{lem:outcome_cost}
    When the cost function $c(x, x')$ is outcome monotonic, then it can be
    written as a cost function over outcome likelihoods $\olcost(l, l')
    \coloneqq c(x, x')$ where $x, x' \in \X$ are any points such that $l =
    \ol(x)$ and $l' = \ol(x')$.
\end{lemma}
	\begin{proof}
		The monotonicity assumptions imply that if $\ol(x^*) = \ol(x')$, then
$c(\cdot, x') = c(\cdot, x^*)$ and $c(x', \cdot) = c(x^*, \cdot)$.  Thus,
$\olcost(l, l') \coloneqq c(x, x')$ is well-defined because any points $x$ and
$x'$ such that $l = \ol(x)$ and $l' = \ol(x')$ yield the same value of $c(x,
x')$.
	\end{proof}

Throughout the paper, we will make use of the equivalent likelihood cost
$\olcost$
when
a
proof
is
more naturally expressed with $\olcost$,
rather
than with
the
underlying
cost $c$.

    \section{Institutional Utility Versus Social Burden}
\label{sec:inst_vs_social}
In this section, we characterize the inherent trade-offs between institutional
utility and social burden in the strategic setting. In particular, we show
any classifier that improves institutional utility over the best classifier in
the static setting causes a corresponding increase in social burden.

To prove this result we first show that any classifier can be
represented as a
threshold classifier that accepts all individuals with outcome likelihood greater than some threshold $\tau \in [0, 1]$. Then,
we show increasing utility for the institution requires raising this
threshold $\tau$, but that this always increases the social burden.

Equipped with this result, we show the (Pareto-optimal) set of classifiers that
increase institutional utility in the strategic setting corresponds to an
interval $I$. Each threshold $\tau \in I$ represents a particular trade-off
between institutional utility and social burden. Strategic classification
corresponds to one extremum: the best strategic utility but the worst social
burden. The non-strategic utility corresponds to the other: doing nothing to
prevent gaming.  Neither is likely to be the right trade-off in practical
contexts. Real domains will require a careful weighting of these two utilities,
leading to a choice somewhere in between. Thus, a main contribution of our work
is exposing this interval.

\subsection{General Trade-Off}

We now proceed to prove the trade-off between institutional utility and social
burden. Our first step is to show that in the strategic setting we can restrict
attention to classifiers that threshold on the outcome likelihood (assuming the
cost is outcome monotonic as in Definition \ref{def:cost}).
\begin{definition}[Outcome threshold classifier]
An outcome threshold classifier $f$ is a classifier of the form $f(x) =
\ind{\ol(x) \geq \tau}$ for $\tau \in [0, 1]$.
\end{definition}

In practice, the institution may not know the outcome likelihood $\ell(x) =
\prob{Y = 1 \mid X = x}$. However, as shown in the next lemma, for any
classifier that they do use, there is a threshold classifier with equivalent
institutional utility and social burden.  Thus, we can restrict our theoretical
analysis to only consider threshold classifiers.
\begin{lemma}
For any classifier $f$ there is an outcome threshold classifier $f'$ such that
$\iutil(f) = \iutil(f')$ and $\scost(f) = \scost(f')$.
\end{lemma}
\begin{proof} 
Let $\tau(f) = \min_{x: f(x) = 1} \ell(x)$ be the minimum outcome likelihood at which
an individual is accepted by the classifier $f$. Then, let
$f'(x) = \ind{\ol(x) \geq \tau(f)}$ be the outcome threshold classifier that
accepts all individuals above $\tau(f)$. We will show that the institutional
utility and social burden of $f$ and $f'$ are equal.
		
Since the cost function is outcome monotonic, it is the same cost to move
to any point with the same outcome likelihood. Furthermore, it is
higher cost to move to points of higher likelihood, i.e, if $\ol(x') >
\ol(x^{*}) > \ol(x)$, then $c(x, x')
> c(x, x^{*}) > 0$. Since individuals game optimally, when an individual changes
her features in response to the classifier $f$, she has no incentive to move to
a point with likelihood higher than $\tau(f)$ -- that would just cost more.
Therefore, she will move to any point with
likelihood $\tau(f)$ to be accepted by $f$ and will incur the same cost,
regardless of which point it is.
Thus, we can write the set of individuals accepted by $f$, $\acc(f)$, as
\begin{align*}
    \acc(f) & = \{x \mid f(\Delta(x)) = 1\} \\
    & = \{x \mid \exists x' : f(x') = 1, ~c(x, x') \leq 1 \} \\
    & = \{x \mid \exists ~x': \ol(x') = \tau(f), ~c(x, x') \leq 1 \} \, .
\end{align*}
		
Since $\tau(f) = \tau(f')$, the individuals accepted by $f$ and $f'$ are
equal: $\acc(f) = \acc(f')$. Therefore, their institutional utilities
$\iutil(f)$ and $\iutil(f')$ are equal. We can similarly show that the
social burdens of $f$ and $f'$ are also equal:
\begin{align*}
\scost(f) & = \ex \left [\min_{x':f(x') = 1} c(x, x') \mid y = 1 \right] \, , \\
          & = \ex \left [c(x, x') \mid y = 1 \right] \quad \text{for some }\;
    x' : \ol(x') = \tau(f), \\
    & = \ex \left [c(x, x') \mid y = 1 \right] \quad \text{for some }\; x' :
    \ol(x') = \tau(f'), \\
& = \ex \left [\min_{x':f'(x') = 1} c(x, x') \mid y = 1 \right] = \scost(f') \, .
\end{align*}
\end{proof}

Since outcome threshold classifiers can represent all
classifiers in the strategic setting, we will henceforth only consider
outcome threshold classifiers. Furthermore, we
will
overload
notation and use $\iutil(\tau)$ and $\scost(\tau)$ to refer to $\iutil(f_\tau)$
and
$\scost(f_\tau)$ where $f_\tau(x) = \ind{\ol(x) \geq \tau}$ is the
outcome threshold classifier with threshold $\tau$.

\begin{figure}
\centering
\includegraphics[width=\columnwidth]{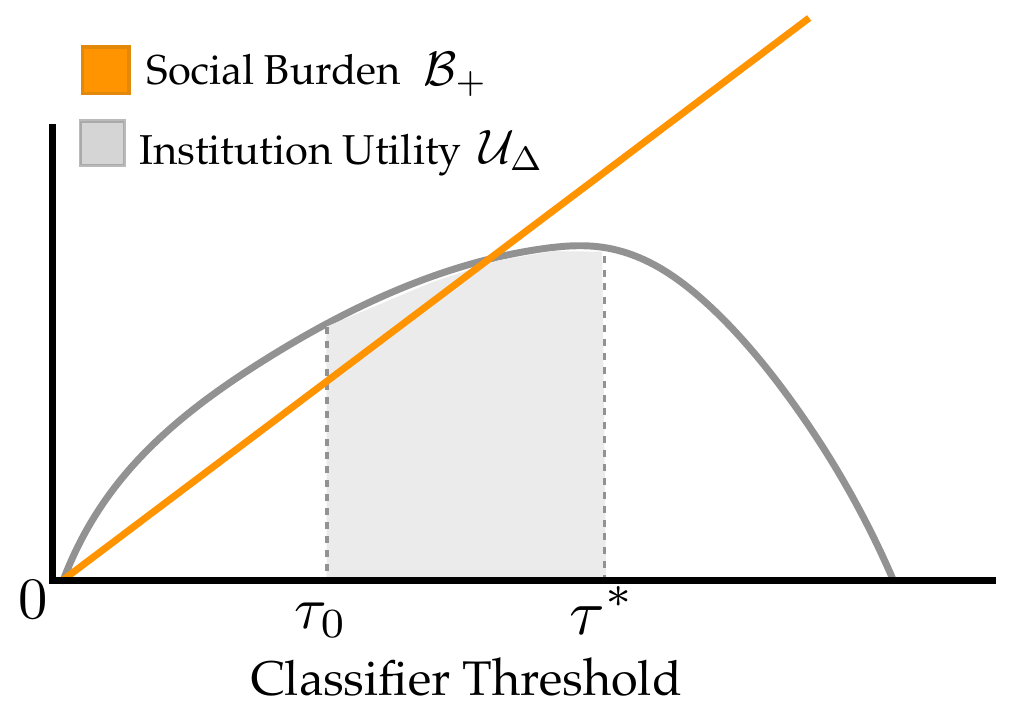}	
\caption{The general shapes of the institution utility and social burden as a
function of the threshold the institution chooses. The threshold $t_0$ is the
non-strategic optimal, while the threshold $\tau^{*}$ is the Stackelberg
equilibrium.}
\label{fig:tradeoff}
\end{figure}

Figure \ref{fig:tradeoff} illustrates how institutional utility and social
burden change as the threshold of the classifier increases. The institutional
utility is quasiconcave, while the social burden is monotonically
non-decreasing. The next lemma provides a formal characterization of the
shapes shown in Figure~\ref{fig:tradeoff}.

\begin{theorem}
    \label{thm:main_tradeoff}
    The institutional utility $\iutil(\tau)$ is quasiconcave in $\tau$ and has a
    maximum at a threshold $\tau^{*} \geq \tau_0$ where $\tau_0=0.5$ is the
    threshold of the non-strategic optimal classifier. The social burden
    $\scost(\tau)$ is monotonically non-decreasing in $\tau$. Furthermore, if
    $\iutil(\tau) \neq \iutil(\tau')$, then $\scost(\tau) \neq \scost(\tau')$.
\end{theorem}
\begin{proof}
Let $\acc(\tau)$ and $\accn(\tau)$ be the set of individuals accepted by
$f$ in the strategic and non-strategic setting, respectively. If $\tau <
\tau_0$, we have $\acc(\tau) \supseteq \acc(\tau_0) \supseteq \accn(\tau_0)$.
Since $\accn(\tau_0)$ is the optimal non-strategic acceptance region, any $x
\not\in \accn(\tau_0)$ has $\ol(x) < 0.5$, and one can increase $\iutil$ by not
accepting $x$, which implies $\iutil(\tau) \leq \iutil(\tau_0)$. 
Therefore, if a threshold $\tau^{*}$ is optimal for the
institution,~i.e.~if $\iutil(\tau^{*}) = \max_{\tau} \iutil(\tau)$, then $\tau^*
\geq \tau_0$. 
		
Recall that a univariate function $f(z)$ is quasiconcave if there exists
$z^*$ such that $f$ is non-decreasing for $z < z^*$ and is
non-increasing for $z > z^*$. Let $\tau^{*}$ be as above. For $\tau_0 < \tau_1 <
\tau^{*}$, we have $\acc(\tau_0) \supseteq \acc(\tau_1)
\supseteq \acc(\tau^{*})$. Since $\iutil(\tau^{*})$ is optimal, $\acc(\tau^{*})$
is the optimal strategic acceptance region, and thus $\iutil(\tau_0) \leq
\iutil(\tau_1) \leq \iutil(\tau^{*})$. Similarly, if $\tau_1 > \tau_0 > \tau^{*}$
we have that $\acc(\tau^*) \supseteq \acc(\tau_0) \supseteq \acc(\tau_1)$, 
and thus $\iutil(\tau_1) \leq \iutil(\tau_0) \leq \iutil(\tau^*)$. Therefore,
$\iutil(\tau)$ is quasiconcave in $\tau$.
		
The individual burden $b_\tau(x) = \min_{\ol(x') \geq \tau} c(x, x')$ is monotonically non-decreasing in $\tau$. Since the social burden $\scost(\tau)$ is equal to $\ex[b_\tau(x) \mid y = 1]$, the social burden is also
monotonically non-decreasing.
		
Suppose $\iutil(\tau) \neq \iutil(\tau')$ and without loss of generality let
$\tau < \tau'$. For all individuals $x$, $b_{\tau'}(x) \geq b_{\tau}(x)$. If
there is at least one individual $x$ such that $b_{\tau'}(x) > b_{\tau}(x)$,
then $\scost(\tau') > \scost(\tau)$. But since $\iutil(\tau) \neq
\iutil(\tau')$, there must exist an individual $x$ such that $x \in
\acc(\tau) / \acc(\tau')$ and $p(X = x \mid Y = 1) > 0$ since $\ol(x) > 0$ 
by assumption. For this individual
$b_{\tau'}(x) > b_{\tau}(x)$. Therefore, $\iutil(\tau) \neq \iutil(\tau')$
implies $\scost(\tau') \neq \scost(\tau)$.
\end{proof}
	
As a corollary, if the institution increases its utility beyond that
attainable by the optimal classifier in the non-strategic setting, then the
institution also causes higher social burden.
\begin{corollary}
Let $\tau$ be any threshold and $\tau_0 = 0.5$ be the optimal threshold in the
non-strategic setting. If $\iutil(\tau) > \iutil(\tau_0)$, then $\scost(\tau) > \scost(\tau_0)$.
\end{corollary}

\subsection{Choosing a Concrete Trade-off}
The previous section shows increases in institutional utility come at a cost in
terms of social burden and vice-versa. This still leaves open the question: what
is the concrete trade-off an institution should choose?

Theorem~\ref{thm:main_tradeoff}
provides
a
precise
characterization
of
the
choices available to trade-off between institutional utility and social burden.
The
baseline choice for the institution is to not account for strategic behavior
 and use the
non-strategic
optimum $\tau_0$. Maximizing utility without regard to
social
burden leads the institution to choose $\tau\opt$. In
general,
the
interval
$[\tau_0,
\tau\opt]$
offers
the set of trade-offs the institution considers. Choosing $\tau > \tau_0$ can
increase robustness at the price of increasing social burden. Thresholds $\tau
>
\tau\opt$ are not Pareto-efficient and are not considered.

Much of the prior work in machine learning has focused exclusively on solutions
corresponding to the
thresholds at the extreme: $\tau_0$ and $\tau\opt$.  The
threshold
$\tau_0$
is the solution when strategic behavior is not accounted for. The
threshold
$\tau\opt$
is also known as the \emph{Stackelberg
equilibrium} and is the subject of recent work in
strategic classification \cite{bruckner2011stackelberg,
hardt2016strategic,
dong2018strategic}. While using
$\tau\opt$ may be warranted in some cases, a proper accounting of social burden
would lead institutions to choose classifiers somewhere \emph{between} the
extremes of $\tau_0$ and $\tau^{*}$.

The exact choice of $\tau \in [\tau_0, \tau^{*}]$ is
\emph{context-dependent}
and
depends on balancing concerns between institutional and broader social
interest.
We now highlight cases where using $\tau_0$ or $\tau\opt$ may be suboptimal,
and
using a threshold $\tau \in (\tau_0, \tau\opt)$ that balances robustness with
social burden is preferable.

\begin{example}[Expensive features.]
If measuring a feature is costly for individuals and offers limited
gains in predictive accuracy, an institution may choose to ignore the feature, even if
it
means giving up accuracy on the margin. In an educational context, a university
may decide to no longer require applicants to submit standardized test scores,
which can cost applicants
hundreds of dollars, if the corresponding improvement in admissions
outcomes is very small
\cite{princeton2018higher}. 
\end{example}

\begin{example}[Reducing social burden under resource constraints.]
Aid organizations increasingly use machine learning to determine where to allocate resources
after natural disasters \citep{imran2014aidr}. In these cases, positive
individuals are precisely those people who are in need of aid and may
experience
very high costs to change their features. Using thresholds with high social
burden
is therefore undesirable. At the same time, aid organizations often
face significant resource constraints. False positives from individuals gaming
the classifier ties up resources that could be better used elsewhere.
Consequently, using the non-strategic threshold is also undesirable.  The aid
organization should choose a some threshold $\tau$ with $\tau_0 < \tau <
\tau^{*}$ that reflects these trade-offs.
\end{example}

\begin{example}[Misspecification of agent model.]
Strategic classification models typically assume the
individual optimally responds to the classifier $f$. In reality,
individuals will not have perfect knowledge of the classifier $f$ when it is
first deployed. Instead, they may be able to learn about how the classifier
works over time, and gradually improve their ability to game the classifier.
For
example, self-published romance authors exchanged information in private chat
groups about how to best game Amazon's book recommendation algorithms
\citep{amazon_romance}. For the institution, it is difficult to a priori model
the
dynamics of how information about the classifier propagates. A preferable
solution may be to simply make the assumption that the individual can best
respond to the classifier, but to only gradually increase the threshold from
the
non-strategic $\tau_0$ to the Stackelberg optimal $\tau^*$ over time.
\end{example}

In fact, misspecification of the agent model (described above), is
why \citet{bruckner2012static} suggest the Stackelberg equilibrium is too
conservative, and instead prefer to use Nash equilibrium strategies.
Complementary to their observation, we show that there is a more general reason
Nash equilibria may be preferable. Namely, that Nash equilibria have lower
social burden than the Stackelberg solution. As the following lemma shows, in
our context, the set of Nash equilibria form an interval $[t_N,
\tau^{*}]
\subset I$ for some $t_N \geq t_0$. The proof is deferred to the appendix.
\begin{lemma}
\label{lem:nash}
Suppose the cost over likelihoods $c_L$ is continuous and $\ol(\X) = [0, 1]$,
i.e, all likelihoods have non-zero support. Then, the set of Nash equilibrium
strategies for the institution is $[\tau_N, \tau^*]$ for some $\tau_N \geq
\tau_0$ where $\tau_0 = 0.5$ is the non-strategic optimal threshold and
$\tau^*$
is the Stackelberg equilibrium strategy.
\end{lemma}

The Stackelberg equilibrium requires the institution to choose $\tau^{*}$,
whereas Nash equilibria give the institution latitude to trade-off between
institutional utility and social burden by choosing from the interval
$[t_N,
\tau^{*}]
\subset I$.
This
provides
an
additional argument
in
favor
of
Nash equilibria-- institutions can still reason in terms of equilibria \emph{and}
achieve more favorable outcomes in terms of social burden.
	
	\section{Fairness to Subpopulations}
\label{sec:subpop}
Our previous section showed that increased robustness in the face of strategic
behavior comes at the price of additional social burden. In this section, we
show this social burden is not fairly
distributed: when the individuals being classified are from latent
subpopulations, say of race, gender, or socioeconomic status, the social burden
can disproportionately fall on disadvantaged subpopulations.  Furthermore, we
find that improving the institution's utility can exacerbate the gap between
the
social burden incurred by an advantaged and disadvantaged group.

Concretely, suppose each individual is from a subpopulation $g \in \{a, b\}$.
The social burden a classifier $f$ has on a group $g$ is the expected
minimum
cost required for a positive individual from group $g$ to be accepted:
$\scostg{g}(f) = \ex \left [\min_{x':f(x')=1} c(x, x') \mid Y = 1, G = g \right
]$.
We can then define the social gap between groups $a$ and $b$:

\begin{definition}[Social gap]
    The \textit{social gap} $\sgap(f)$ induced by a classifier $f$ is the
    difference in the social burden to group $b$ compared to $a$: $\sgap(f) =
    \scostg{b}(f) - \scostg{a}(f)$.
\end{definition}

The social gap is a measure of how much more costly it is for a positive
individual from group $b$ to be accepted by the classifier than a positive
individual from group $a$. For example, there is evidence that women need to
attain higher educational qualifications than their male counterparts to
receive
the same salary \cite{carnevale2018women}. 

A high social gap is alarming for two reasons. First, even when two people from
group $a$ and group $b$ are equally qualified, the individual from group $a$
may
choose not to participate at all because of the cost she would need to endure
to
be accepted. Secondly, if she does decide to participate, she may continue to
be
at a disadvantage after being accepted because of the additional cost she had
to
endure, e.g., repaying student loans.

Non-strategic classification can already induce a social gap between two
groups, and strategic classification can exacerbate this gap.
We show this under two natural ways group $b$ may be
disadvantaged. In the first setting, the feature distributions of group $a$ and
$b$ are such that a positive individual from group $b$ is less likely to be
considered positive, compared to group $a$. In the second setting, individuals
from group $b$ have a higher cost to adapt their features compared to group
$a$.
Under both of these conditions, any improvement the institution can make to its
own strategic utility has the side effect of worsening (increasing)
the social gap.
	
\subsection{Different Feature Distributions}
\label{sec:diff-feats}
In the first setting we analyze, the way groups $a$ and $b$ differ is through
their distributions over features. We say that group $b$ is disadvantaged if
the
features distributions are such that positive individuals from group $b$ are
less likely to be considered positive than those from group $a$. Formally, this
can be characterized as the following:

\begin{definition}[Disadvantaged in features] \label{def:feats}
    \label{def:fosd}
    Let $\olrv{g} = \ol(X) \mid Y = 1, G = g$ be the outcome likelihood of a
    positive individual from group $g$, and let $\olcdf{g}$ be the cumulative
    distribution function of $\olrv{g}$. We say that group $b$ is disadvantaged
    in features if $\olcdf{b}(l) > \olcdf{a}(l)$ for all $l \in (0, 1)$.
\end{definition}

In the economics literature, the relationship between $\olrv{a}$ and $\olrv{b}$
is referred to as \textit{strict first-order stochastic dominance}
\cite{levy1992stochastic}. Intuitively,
that group $b$ is disadvantaged in features if and only if the distribution of
$\olrv{a}$ can be transformed to the distribution of $\olrv{b}$ by transferring
probability mass from higher values to lower values. This definition captures
the notion that the outcome likelihood of positive individuals from group $b$
is
skewed lower than the outcome likelihood of positive individuals from $a$.

In a case study on FICO credit scores in Section~\ref{sec:experiments}, we find
the minority group (blacks) is disadvantaged in features compared to the
majority group (whites) (see Figure \ref{fig:feat_disadvantage}). There are
many
reasons that a group could be disadvantaged in features. Below, we go through a
few potential causes.

\begin{example}[Group membership explains away features]
    Even if two groups are equally likely to have positive individuals, i.e.,
    $\prob{Y = 1 \mid G = a} = \prob{Y = 1 \mid G = b}$, group $b$ can still be
    disadvantaged compared to group $a$. Consider the graph below. Although the
    label $Y$ is independent of the group $G$, the label $Y$ is not independent
    of the group $G$ once conditioned on the features $X$ because the group $G$
    can provide an alternative reason for the observed features.
	\begin{center}
		\includegraphics[scale=0.4]{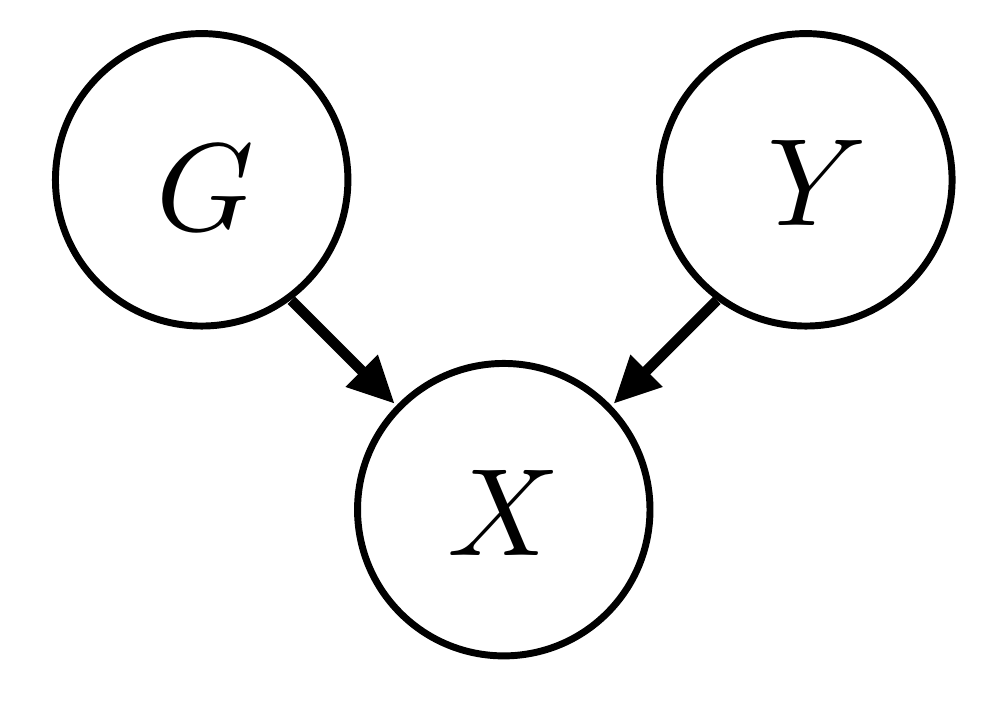}	
	\end{center}
    Concretely, let groups $a$ and $b$ be native and non-native speakers of
    english, $X$ be the number of grammatical errors on an individual's job
    application, and $Y$ be whether the individual is a qualified candidate.
    Negative individuals ($Y = 0$) are less meticulous when filling out their
    application and more likely to have grammatical errors. However, for
    individuals from group $b$ there is another explanation for having
grammatical errors -- being a non-native speaker. Thus, positive individuals
    from group $b$ end up with lower outcome likelihoods than those from $a$,
    even though they may be equally qualified.
\end{example}
	
\begin{example}[Predicting base rates]
    Suppose the rate of positives in group $b$ is lower than that of group $a$:
    $\prob{Y = 1 \mid G = b} < \prob{Y = 1 \mid G = a}$. If there is a feature
in the dataset that can be used as a proxy for predicting the group, such as
    zip code or name for predicting race, then the outcome likelihoods of
    positive individuals from group $b$ can end up lower than those of positive
    individuals from group $a$ because the features are simply predicting the
    base rate of each group.
\end{example}

\paragraph{\textbf{Social gap increases.}} 
We now state and prove the main result showing that the social gap increases as
the
institution increases its threshold for acceptance. Before turning to the
result, we introduce one technical requirement. The \emph{likelihood condition}
is that $\frac{\partial \olcost(l, \tau)}{\partial l}$ is monotonically
non-increasing in $\tau$ for $l, \tau \in [0, 1]$. When the cost function $c$ is
outcome monotonic, the likelihood condition is satisfied for a broad class of
differentiable likelihood cost functions $\olcost$, such as the following
examples.

\begin{itemize}
	\item Differentiable separable cost functions of the form $\olcost(l, l') =
\max(c_2(l') - c_1(l), 0)$ for $c_1, c_2 : [0, 1] \rightarrow
\R_{\geq 0}$.
	\item Differentiable shift-invariant cost functions of the form
	\begin{align*}
	\olcost(l, l') = \begin{cases}
 		c_0(l' - l) & l < l' \\
 		0 & l \geq l'
 	\end{cases} \, ,
	\end{align*}
	for convex $c_0 : [0, 1] \rightarrow \R_{\geq 0}$.
\end{itemize}
 Notably, any linear cost $\olcost(l, l') = \max(\alpha (l'-l), 0)$ where
$\alpha > 0$ satisfies the likelihood condition.  

Under the likelihood condition, we now show that the social gap increases as the
institution increases its threshold for acceptance.
	
	\begin{theorem}
		\label{thm:fosd}
Let $\tau \in (0, 1]$ be the threshold of the classifier. If group $b$ is
disadvantaged in features compared to group $a$, and $\frac{\partial \olcost(l,
\tau)}{\partial l}$ is monotonically non-increasing in $\tau$, then
$\sgap(\tau)$ is positive and monotonically increasing over $\tau$.
\end{theorem}
\begin{proof}
By Lemma~\ref{lem:outcome_cost}, any outcome monotonic cost function can be
written as a cost over outcome likelihoods. Therefore, the social burden can be
written as
\begin{align*}
    \scostg{g}(\tau)
    &= \ex \left [c_L(\ol(x), \tau) \mid Y = 1, G = g \right ]\\
    &= \int_{0}^\tau \olcost(l, \tau) \leibd \olcdf{g}(l),
\end{align*}
where $\olcdf{g}$ denotes the CDF of the group outcome likelihood $\olrv{g}$.
Integrating by parts, we obtain a simple expression for $\scostg{g}(\tau)$:
\begin{align*}
    \scostg{g}(\tau)
    &= \int_{0}^\tau \olcost(l, \tau) \leibd \olcdf{g}(l) \\
	&= [\olcost(l, \tau)\olcdf{g}(l)]^{\tau}_0 - \int_{0}^\tau \frac{\partial
        \olcost(l, \tau)}{\partial l}\olcdf{g}(l) \leibd l \\
    &= - \int_{0}^\tau \frac{\partial \olcost(l, \tau)}{\partial l}
    \olcdf{g}(l) \leibd l \, ,
\end{align*}
where the last line follows because $\olcost(\tau, \tau) = 0$ and $\olcdf{g}(0)
=
0$. This expression for $\scostg{g}(\tau)$ allows us to conveniently write the
social
gap as
\begin{align*}
    \sgap(\tau) 
    = \scostg{b}(\tau) - \scostg{a}(\tau) 
    = \int_{0}^\tau \frac{\partial \olcost(l, \tau)}{\partial l} (\olcdf{a}(l)
    - \olcdf{b}(l)) \leibd l \, .
\end{align*}
We now argue the social gap $\sgap(\tau)$ is positive. By the monotonicity
assumptions, $\frac{\partial \olcost(l, \tau)}{\partial l} < 0$ for $l \in (0,
\tau)$. Since group $b$ is disadvantaged in features, $\olcdf{a}(l) -
\olcdf{b}(l) < 0$ for $l \in (0, 1)$. Therefore, $\sgap(\tau) > 0$.

Now, we show $\sgap(\tau)$ is increasing in $\tau$.  Let $0 \leq \tau < \tau'
\leq 1$. Then, the difference in the social gap is given by
\begin{align*}
    \sgap(\tau') - \sgap(\tau)
    &= \int_0^{\tau} \frac{\partial \left(\olcost(l, \tau') -
    \olcost(l, \tau)\right)}{\partial l} (\olcdf{a}(l)
    - \olcdf{b}(l)) \leibd l \\
    & + \int_\tau^{\tau'} \frac{\partial \olcost(l, \tau')}{\partial l}
(\olcdf{a}(l)
    - \olcdf{b}(l)) \leibd l.
\end{align*}
Since group $b$ is disadvantaged in features, $(\olcdf{a}(l)
- \olcdf{b}(l)) < 0$ for all $l$. By assumption,
$\frac{\partial \olcost(l, \tau)}{\partial l}$ 
is monotonically non-increasing in $\tau$, so the first term is non-negative.
Similarly, $\frac{\partial \olcost(l, \tau')}{\partial l} < 0$ by monotonicity,
so the second term is positive. Hence, $\sgap(\tau') - \sgap(\tau) > 0$, which
establishes $\sgap(\tau)$ is monotonically increasing in $\tau$.
\end{proof}
	
As a corollary, if the institution improves its utility beyond the
non-strategic
optimal classifier, then it also causes the social gap to increase.
	\begin{corollary}
Suppose group $b$ is disadvantaged in features compared to group $a$, and
$\frac{\partial \olcost(l, \tau)}{\partial l}$ is monotonically non-decreasing
in $\tau$. Let $\tau \in (0, 1]$ be a threshold and $\tau_0=0.5$ be the optimal
non-strategic threshold. If $\iutil(\tau) > \iutil(\tau_0)$, then $\sgap(\tau)
> \sgap(\tau_0)$.
	\end{corollary}
	\begin{proof}
By Theorem \ref{thm:main_tradeoff}, if $\iutil(\tau) > \iutil(\tau_0)$, then
$\tau > \tau_0$. By Theorem \ref{thm:fosd}, if $\tau > \tau_0$, then
$\sgap(\tau) > \sgap(\tau_0)$.
	\end{proof}
	
\subsection{Different Costs}
\label{sec:differing-costs}
In Section \ref{sec:diff-feats}, we showed that when two subpopulations have
different feature distributions, the social burden can disproportionately fall
on one group. In this section, we show that even if the feature distributions
of the two groups are exactly identical, the social burden can still
disproportionately impact one group.

We have thus far assumed the existence of a cost function $c$
that is uniform across groups $a$ and $b$. For a variety of structural reasons,
it is unlikely this assumption holds in practice. Rather, it is often the case
that \emph{different groups experience different costs} for changing their
features.

When the cost for group $b$ is systematically higher than the cost for group
$a$, we prove group $b$ incurs higher social burden than group $a$.
Furthermore, if the institution improves its utility by increasing its threshold
$\tau$, then as a side effect it also increases the social gap between group $b$
and $a$ (Theorem \ref{thm:cost-gap}). 

Much of the prior work on fairness in classification focuses on preventing
unfairness that can arise when different subpopulations have different
distributions over
features and
labels \cite{dwork2012fairness,hardt2016equality,chouldechova2017fair}.
Our
result provides a reason to be concerned about the unfair impacts of
a classifier even when two
groups have identical initial distributions. Namely, that \emph{it can be
easier for one group to game the classifier than another}.

Formally, we say that group $b$ is disadvantaged in cost compared to group
$a$ if the following condition holds.
\begin{definition}[Disadvantaged in cost]
Let $c_g(x, x')$ be the cost for an individual from group $g$ to adapt their
    features from $x$ to $x'$. Group $b$ is disadvantaged in cost if $\bcost(x, x') = \kappa \acost(x, x')$ for all $x, x' \in X$ and some scalar $\kappa > 1$.
\end{definition}

Next, we give a variety of example scenarios of when a group can be
disadvantaged in cost.

\begin{example}[Opportunity Costs]
Many universities have adopted gender-neutral policies that stop the
``tenure-clock'' for a year for family-related reasons, e.g. childbirth.
Ostensibly, no research is expected while the clock is stopped.  However, the
adoption of gender-neutral clocks actually \textit{increased} the gap between
the percentage of men and women who received tenure \cite{antecol2016equal}. The
suggested cause is that women still shoulder more of the burden of bearing and
caring for children, compared to men. Men who stop their tenure clock are more
productive during the period than women, who have a higher opportunity cost to
doing research while raising a child.
\end{example}

\begin{example}[Information Asymmetry]
A large portion of high-achieving, low-income students do
not apply to selective colleges, despite the fact that
these colleges are typically \emph{less} expensive for them 
because of the financial
aid they would receive \cite{hoxby2012missing}. This
phenomenon seems to be due to low-income students having less access
to information about college \cite{hoxby2013expanding}. 
Since low-income students have more barriers to gaining
information about college, it is natural to assume that, compared to their 
wealthier peers, they have a higher cost to strategically manipulating their
admission features.
\end{example}

\begin{example}[Economic Differences]
Consider a social media company that wishes to classify individuals as
``influencers,'' either to more widely disseminate their content or to identify
promising accounts for online marketing campaigns. Wealthy individuals can
purchase followers or likes, whereas other groups have to increase
these numbers organically \cite{chafkin_2016}. Consequently, the costs to
increasing one's popularity metric differs based on access to capital.
\end{example}

Finally, our main technical result shows that even when the distributions of
groups $a$ and $b$ are identical, if group $b$ is disadvantaged in cost, then
when the institution increases its threshold for acceptance, it also increases
the social gap between the two groups.

\begin{theorem}
	\label{thm:cost-gap}
    Suppose positive individuals from groups $a$ and $b$ have the same
distribution over features, i.e, if $Z = \left(X \mid Y = 1\right)$, then $Z$ is
independent of the group $G$. If group $b$ is disadvantaged in cost compared
    to group $a$, then the social gap $\sgap(\tau)$ is non-negative and
    monotonically non-decreasing in the threshold $\tau$.
\end{theorem}
\begin{proof}
Since $X \mid Y = 1$ is independent of $G$, the social burden to a group $g$ can be written as
\begin{align*}
\scostg{g}(\tau) = \int_{\X} \min_{x':f_\tau(x')=1} c_g(x, x')p(X = x
\mid Y = 1) \leibd x
\end{align*}
where $f_\tau$ is the outcome likelihood classifier with threshold $\tau$. The social gap can then be
expressed as
\begin{align*}
    \sgap(\tau) & = \scostg{b}(\tau) - \scostg{a}(\tau) \\
    & = \int_{\X}
 (\kappa - 1) \min_{x':f_\tau(x')=1} c_a(x, x') ~p(X = x \mid Y = 1) \leibd x \, \\
& = (\kappa - 1) \scostg{a}(\tau) .
\end{align*}
Since the group social burden $\scostg{a}(\tau)$ is non-negative and monotonically non-decreasing, 
the social gap $\sgap(\tau)$ is also non-negative and monotonically non-decreasing.
\end{proof}

    \section{Case Study: FICO Credit Data}
\label{sec:experiments}
We illustrate the impact of strategic classification on different subpopulations
in the context of credit scoring and lending. FICO scores are widely used in the
United States to predict credit worthiness. The scores themselves are derived
from a proprietary classifier that uses features that are susceptible to gaming
and strategic manipulation, for instance the number of open bank accounts.

We use a sample of 301,536 FICO scores derived from TransUnion TransRisk scores
\citep{fed2007report} and preprocessed by \citet{hardt2016equality}. The scores
$X$ are normalized to lie between 0 and 100. An individual's outcome is labeled
as a
\emph{default} if she failed to pay a debt for at least 90 days on at least one
account in the ensuing 18-24 month period. Default events are labeled with
$Y=0$, and otherwise repayment is denoted with $Y=1$. The two subpopulations are
given by race: $a=\text{\tt white}$ and $b=\text{\tt black}$. 

We assume the credit lending institution accepts individuals based on a
threshold on the FICO score. Using the normalized scale, a threshold of
$\tau=58$ is typically used to determine eligibility for prime rate loans
\cite{hardt2016equality}. Our results thus far have used thresholds on the
outcome likelihood, rather than a score. However, as shown in
Figure~\ref{fig:repay_probs}, the outcome likelihood is monotonic in the FICO
score. Therefore, all our conditions and results can be validated using the
score instead of the outcome likelihood.

\subsection{Different Feature Distributions}
In Section~\ref{sec:diff-feats}, we studied the scenario where the distribution
of outcome likelihoods $\ol(X) = \prob{Y = 1 \mid X}$ differed across
subpopulations. In particular, if the likelihoods of the positive individuals in
group $B$ tend to be lower than the positive individuals in group $A$, then
increasing strategic robustness increases the social gap between $A$ and $B$. 

Interestingly, such a skew in score distributions exists in the FICO
data. Black borrowers who repay their loans tend to have lower FICO scores than
white borrowers who repay their loans. In terms of the corresponding score CDFs,
for every score $x$, $\olcdf{\tt black}(x) \geq \olcdf{\tt white}(x)$.
Figure~\ref{fig:feat_disadvantage} demonstrates this observation.

\begin{figure}
    \centering
    \includegraphics[width=\columnwidth]{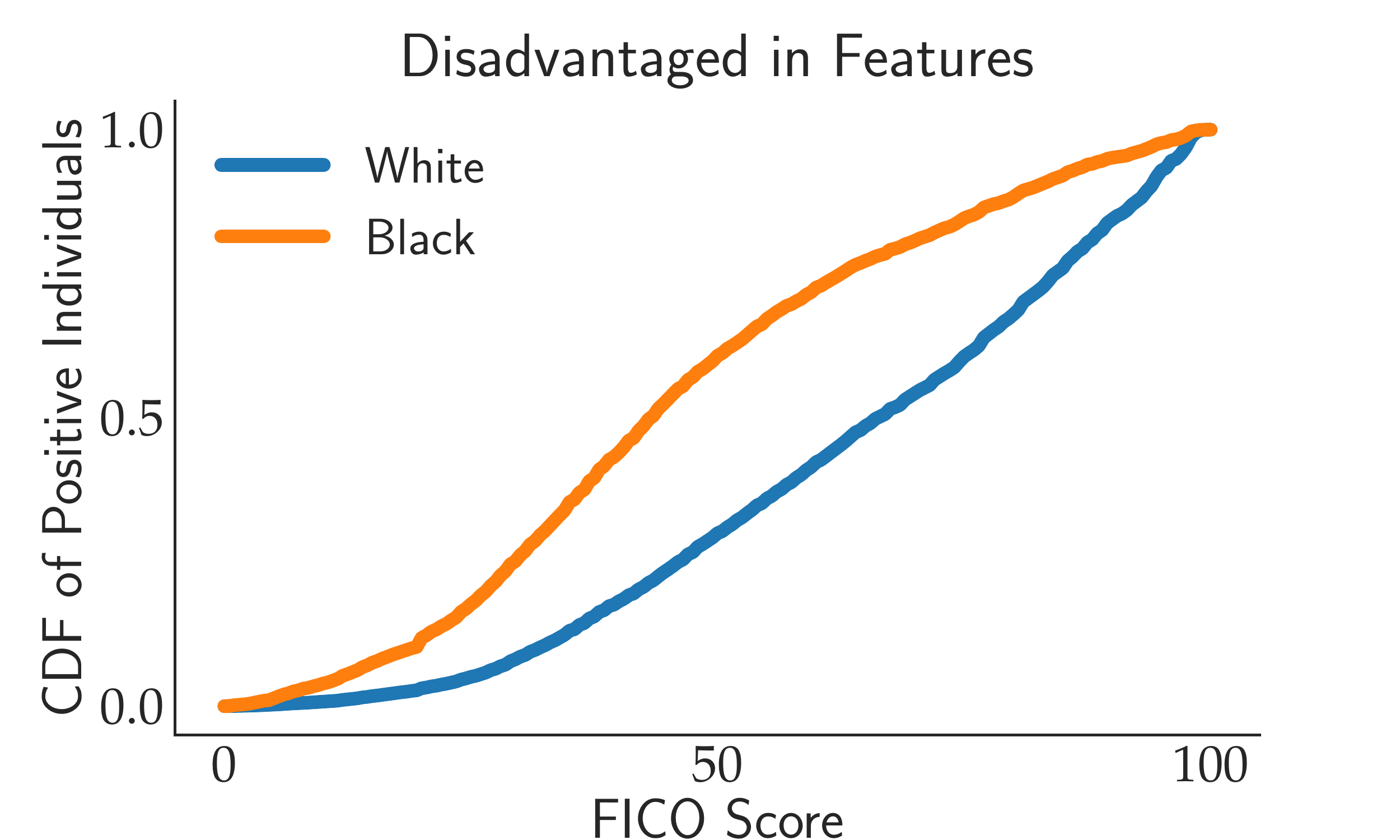}
    \caption{
        Comparison of the distribution of FICO scores among black and white
        borrowers who repaid their loans. Credit-worthy black individuals tend
        to have lower credit scores than credit-worthy white individuals. The
        comparison of the corresponding CDFs demonstrates our 
        ``disadvantaged in features'' assumption holds.}
    \label{fig:feat_disadvantage}
\end{figure}

When the score distribution among positive individuals is skewed,
Theorem~\ref{thm:fosd} guarantees the social gap between groups is increasing in
the threshold under a reasonable cost model.  Operationally, raising the loan
threshold to protect against strategic behavior increases the relative burden on
the black subgroup. To demonstrate this empirically, we use a coarse linear cost
model, $c(x, x') = \max(\alpha(x'- x), 0)$ for some $\alpha > 0$. Here, $\alpha$
corresponds to the cost of raising one's FICO score one point.
Since the probability of repayment $\prob{Y=1\mid x}$ is monotonically increasing in $x$,
the linear cost $c$ satisfies the requisite outcome monotonicity conditions.  

\begin{figure}
    \centering
    \includegraphics[width=\columnwidth]{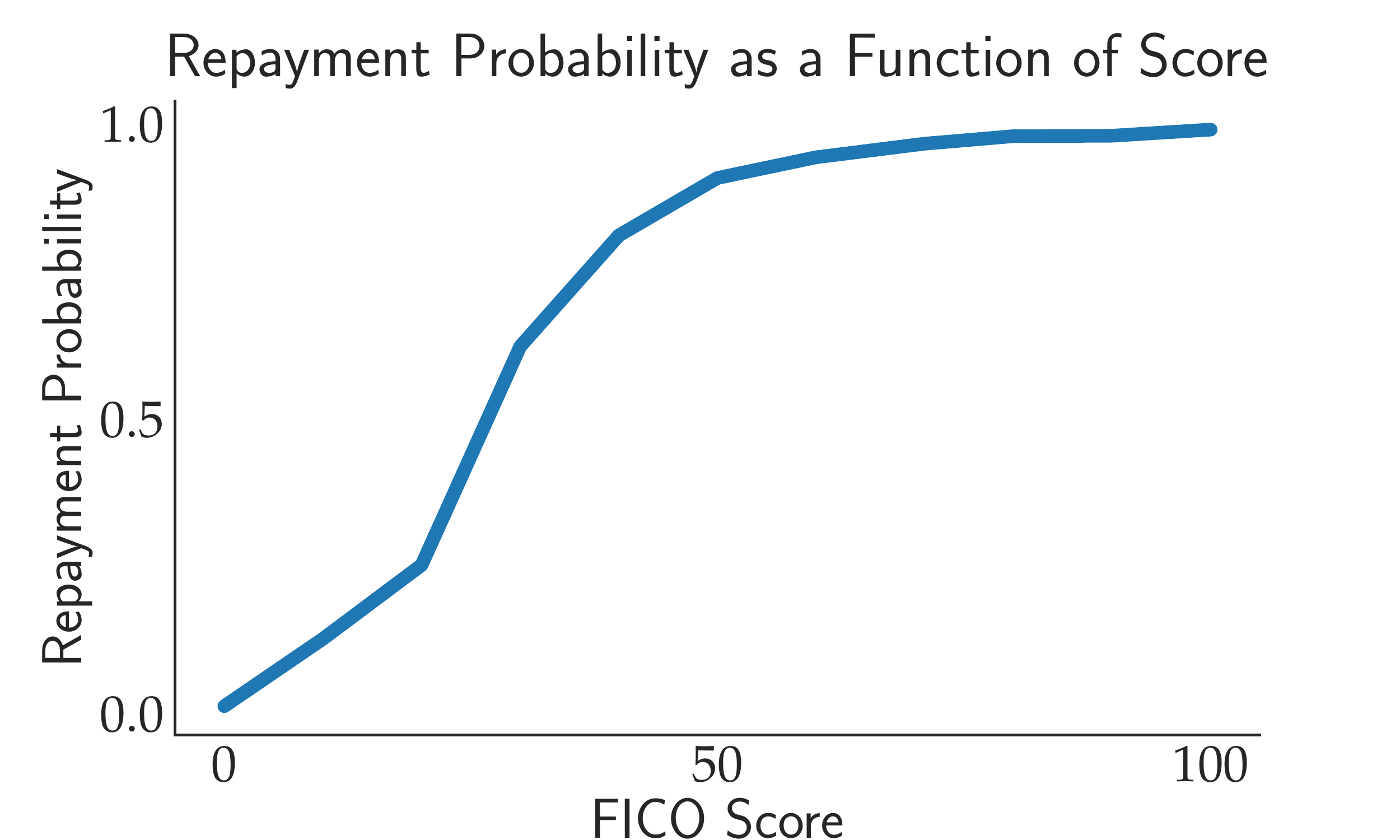}
    \caption{
        Repayment probability as a function of credit score. Crucially, the
probability of repayment $\prob{Y=1 \mid x}$ is monotonically increasing in
        $x$.}
    \label{fig:repay_probs}
\end{figure}

In Figure~\ref{fig:social_gap_curves_same_costs}, we compute $\sgap(\tau)$ as
$\tau$ varies from $0$ to $100$ for a range of different value of $\alpha$. For
any $\alpha$, the social utility gap is increasing in $\tau$.  Moreover, as
$\alpha$ becomes large, the rate of increase in the social gap grows large as
well.

\begin{figure}
    \centering
    \includegraphics[width=\columnwidth]{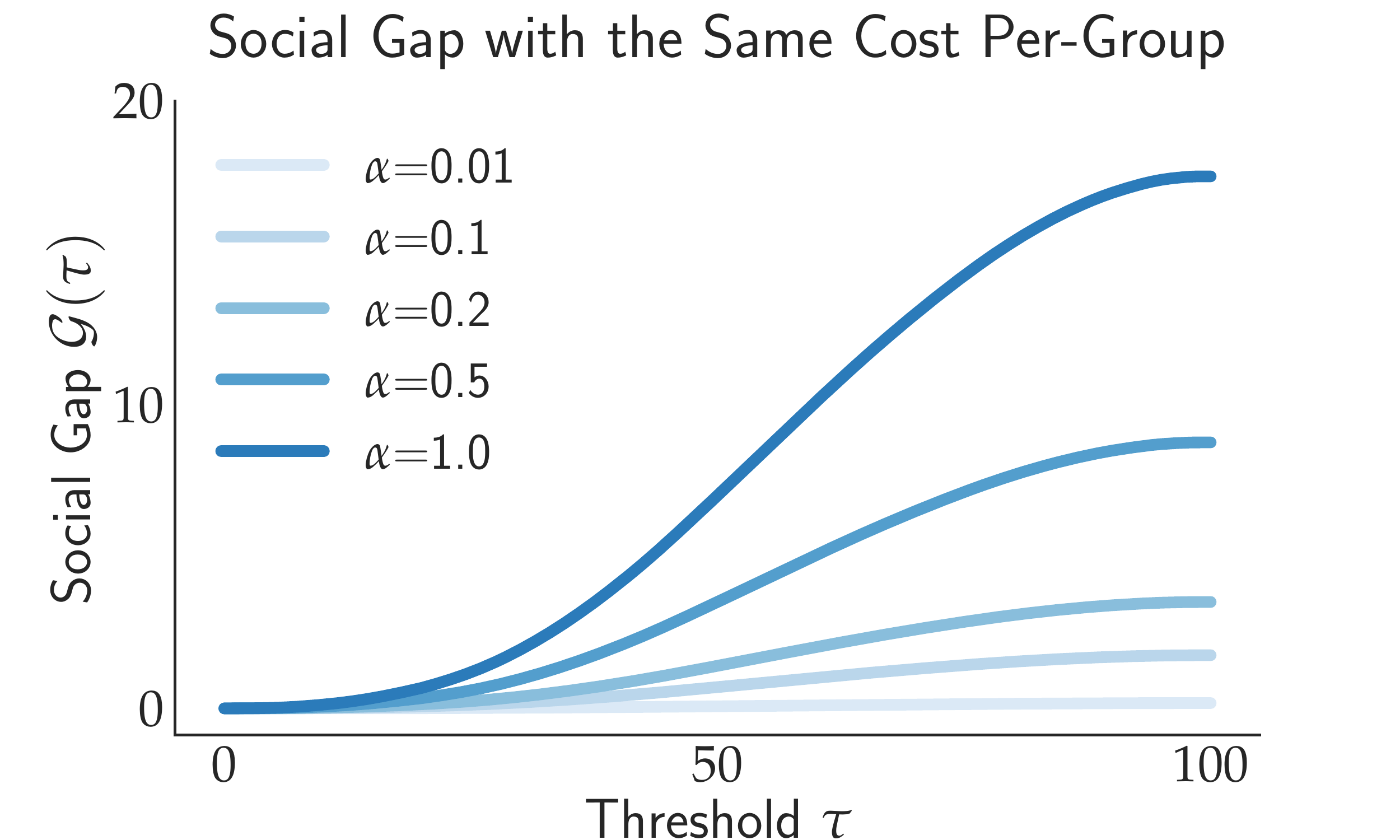}
    \caption{
        Impact of increasing the threshold $\tau$ on white and black credit
        applicants. When the cost to changing one's score $\alpha$ is small,
        increases to the threshold have only a small effect on the social gap.
        However, as $\alpha$ becomes large, even small increases to the
        threshold can create large discrepancies in social burden between the two
        groups.}
    \label{fig:social_gap_curves_same_costs}
\end{figure}

\subsection{Different Cost Functions}
In Section~\ref{sec:differing-costs}, we demonstrated when two subpopulations
are identically distributed, but incur different costs for changing their
features, there is a non-trivial social gap between the two. In the context of
the FICO scores, it is plausible that blacks are both disadvantaged in features
\emph{and} experience higher costs for changing their scores. For instance,
outstanding debt is an important component of FICO scores. One way to reduce
debt is to increase earnings. However, a persistent black-white wage gap between
the two subpopulations suggest increasing earnings is easier for group $a$ than
group $b$ \cite{daly2017disappointing}. This setting is not strictly captured
by
our existing results, and we should expect the effects of both different costs
functions and different feature distributions to compound and exacerbate the
unfair impacts of strategic classification.

To illustrate this phenomenon, we again use a coarse linear cost model.
Suppose group A has cost $c_A(x, x') = \max\{\alpha(x' - x), 0\}$ for
some $\alpha > 0$, and group B has cost $c_B(x, x') = \max\{\beta(x' - x), 0\}$
for some $\beta \geq \alpha$. As in Section~\ref{sec:differing-costs}, group B
is disadvantaged in cost provided the ratio $\kappa = \beta / \alpha > 1$.
In Figure~\ref{fig:social_gap_curves_diff_costs}, we show the social gap
$\sgap(\tau)$ for various settings of $\kappa$. The social gap is
always increasing as a function of $\tau$, and the rate of increase grows large
for even moderate values of $\kappa$. When $\kappa$
is large, even small increases in $\tau$ can disproportionately increase the
social burden for the disadvantaged subpopulation.

\begin{figure}
    \centering
\includegraphics[width=\columnwidth]{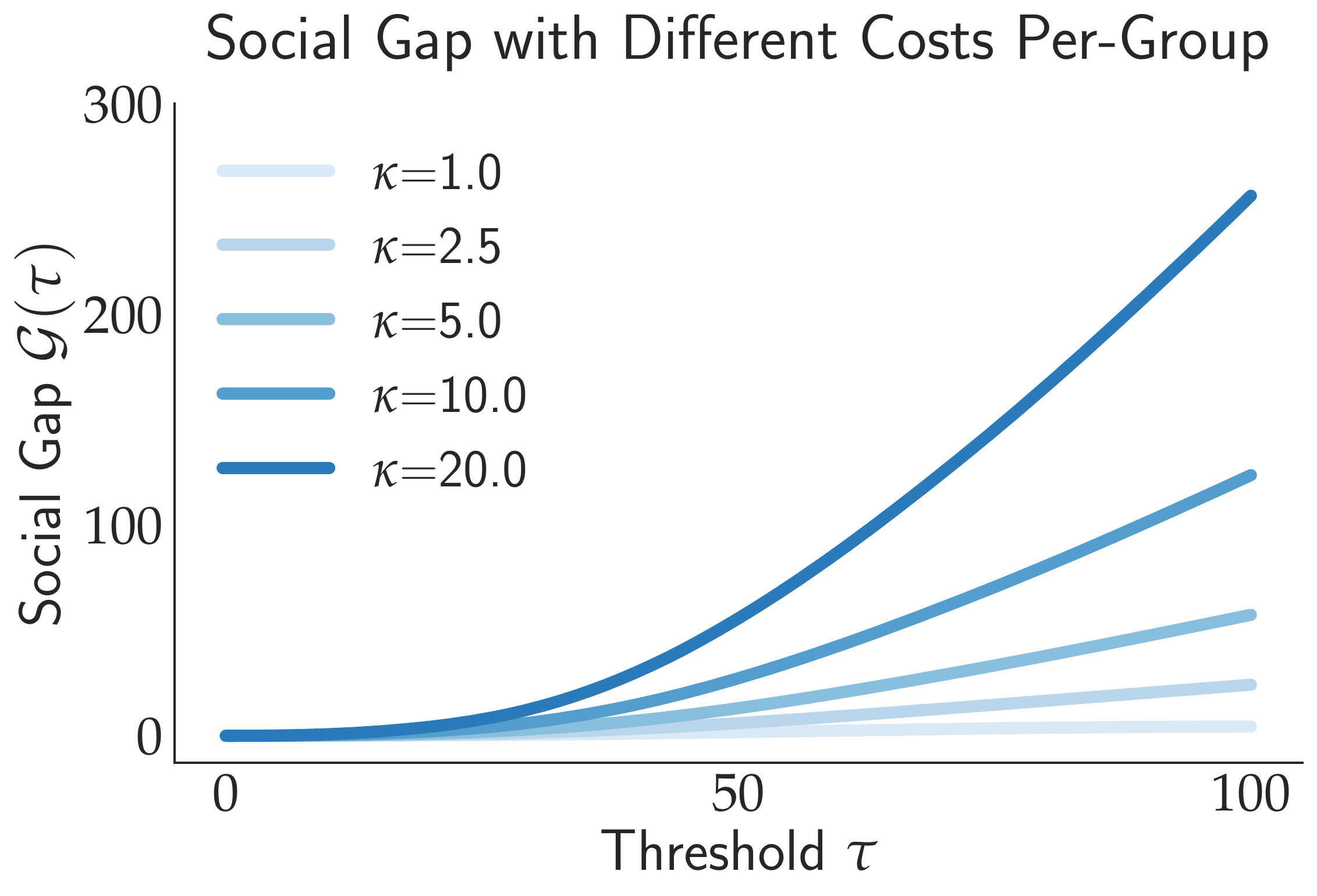}
    \caption{
        Impact of increasing the threshold $\tau$ on white and black credit
        applicants, under the assumption that both groups incur different costs
        for increasing their credit score. As the ratio between the costs
        $\kappa$ increases, the social cost gap grows rapidly
        between the two groups.}
    \label{fig:social_gap_curves_diff_costs}
\end{figure}

	\section{Related Work}
\label{sec:related_work}
\paragraph{Strategic Classification}
Prior work on strategic classification focuses solely on the institution,
primarily aiming to create high-utility solutions for the institution.
Our
work,
on
the
other
hand,
studies
the
tradeoff
between
the
institution's utility and the burden to the individuals being
classified.

~\citet{hardt2016strategic,
dong2018strategic, bruckner2011stackelberg} give algorithms to compute the
Stackelberg equilibrium, which corresponds to the extreme $\tau^{*}$ solution
in our trade-off curves. Although the Stackelberg equilibrium leads to maximal
institutional utility, we show that it also causes high social burden. We 
give several examples of when the high social burden induced by the Stackelberg
equilibrium makes it an undesirable solution for the institution.

Rather than the Stackelberg equilibrium, others have also considered finding
Nash equilibria of the game ~\citep{bruckner2012static, dalvi2004adversarial}.
 ~\citet{bruckner2012static} argue that since in practice people cannot
optimally respond to the
classifier, the Stackelberg solution tends to be too
conservative, and thus a Nash equilibrium strategy is preferable. 
Our
work
provides
a
complementary
reason
to
prefer
Nash
equilibria over the Stackelberg solution. Namely, for a broad class
of cost functions, any Nash equilibrium that is not equal to the
Stackelberg equilibrium places lower social burden on individuals.

Finally, we focus on the setting where individuals
are merely ``gaming'' their features, i.e., they do not
improve their true label by adapting their features. However, if the classifier
is able to incentivize strategic behavior that helps improve negative
individuals, then the social burden placed on positive individuals
may be considered acceptable.
~\citet{kleinberg2018classifiers} studies how to design classifiers that produce
such incentives.

\paragraph{Fairness}
Our work studies how \emph{strategic classification} results in
differing impacts
to different subpopulations and
is complementary to the large body of work studying the
differing impacts of \emph{classification}
\citep{executive2016big,barocas2016big}.

The prior work on classification is primarily concerned with preventing
unfairness
that
can arise due to subpopulations having differing distributions over features or
labels
\cite{hardt2016equality,dwork2012fairness,chouldechova2017fair}. We show
that in the
strategic
setting, a classifier can have differing impact due to the subpopulations
having
differing
distributions \emph{or} differing costs to adapting their features. Therefore,
when individuals are strategic, our work provides an additional reason to
be
concerned
about the fairness of a classifier. In particular, it can be easier for
one group to game the classifier than another.

Furthermore, we
show
that
if
the
institution
modifies
the
classifier
it
uses
to
be
more
robust to strategic behavior, then it also
as a side effect,
increases the gap between the
cost incurred by a disadvantaged subpopulation and an advantaged population.
Thus,
\emph{strategic classification can exacerbate unfairness in
classification.}

Our work is also complementary to \citet{liu2018delayed}, who also
analyze
how
the institution's utility trades-off with the impact to individuals. They study
the trade-off in the non-strategic setting and measure the impact of a
classifier using a dynamics model of how individuals are affected by the
classification
they
receive.
We
study
the
tradeoff in the strategic setting and
measure
the
impact of a
classifier by the cost of the strategic behavior induced by the classifier.

In concurrent work, \citet{hu2018disparate} also study negative externalities
of strategic classification.
In
their
model, they show that
the
Stackelberg equilibrium leads to only false
negative errors on a disadvantaged population and false positives on
the advantaged population. Furthermore, they show that providing a cost subsidy
for disadvantaged individuals can lead to worse outcomes for everyone.

	\section{Discussion of Social Burden} \label{sec:disc}
To measure the impact of strategic classification on the individuals
being classified, we introduced a measure of \emph{social burden}, defined
as
the
expected
cost
that
positive
individuals need to incur to be classified positively: $\scost(f) = \ex \left
[\min_{f(x') = 1} c(x, x') \mid y = 1 \right]$. An alternative measure one
might consider is the expected individual
utility for the positives: $\sutil(f) = \ex
\left
[
u_x(f, \Delta(x)) \mid y = 1
\right ]$, which we will denote the \emph{social utility}.

We prefer social burden to social utility because it makes fewer assumptions
about individual behavior. Social utility measures the utility of the individual
while assuming that they respond optimally and needs the assumption to hold to be a
meaningful measure. Social burden, on the other hand, applies
irrespective of the different policies individuals may actually act according
to. Our analysis assumes \emph{the institution} assumes individuals
respond optimally, but we ourselves believe this to be a strong assumption to hold in practice, and would like our measure of impact on individuals to apply regardless.

Moreover, most of our results are agnostic to the specific choice of social cost
measure.  The results in Section \ref{sec:inst_vs_social} about the tradeoff
between institutional utility and social burden all still hold. Specifically,
Theorem \ref{thm:main_tradeoff} holds with social utility instead of social
burden (and monotonically non-increasing instead of monotonically non-decreasing
since lower utility is worse). For our results in Section \ref{sec:subpop}, i.e,
Theorems \ref{thm:fosd} and \ref{thm:cost-gap}, there is still always a
non-negative social gap (now defined as the difference in social utilities
between the groups), but it is not necessarily true that the social gap
increases as the institution's threshold increases.

While both social burden and social utility apply only to positive individuals,
one could also use versions that integrate over all individuals:
$\mathcal{B}(f)
=
\ex \left
[\min_{f(x') = 1} c(x, x') \right]$ and $\mathcal{S}(f) = \ex
\left
[
u_x(f, \Delta(x))
\right ]$. Our results for $\scost(f)$ go through for $\mathcal{B}(f)$, and the
results for $\sutil(f)$ go through for $\mathcal{S}(f)$\footnote{For the results
in Section \ref{sec:diff-feats} the disadvantaged in features condition defined in Definition
\ref{def:feats} should be modified to no longer condition on $Y=1$.}.
However, in many cases giving a positive classification (e.g. a loan) to a negative
individual (someone who will default) can result in a long-term
negative impact to that individual \cite{liu2018delayed}. In general, it is
uncertain whether the reducing the costs incured by the negative individuals
confers positive social benefits, and we do not incorporate these
costs into our measure.

Overall, there are many potential measures that are complementary to our
measure
of
social burden, but they all provide a similar takeaway. Namely, that in
the strategic setting, there
is
a
tradeoff between institutional accuracy and
individual impact that
must
be
considered
when making choices about strategy-robustness.

	\section{Acknowledgements}
This material is based upon work supported by the National Science Foundation
Graduate Research Fellowship Program under Grant No. DGE 1752814. Any opinions,
findings, and conclusions or recommendations expressed in this material are
those of the author(s) and do not necessarily reflect the views of the National
Science Foundation.

	\bibliography{refs.bib}
	\bibliographystyle{ACM-Reference-Format}
	
	\newpage
\begin{appendices}
	\section{Proof of Lemma \ref{lem:nash}}
	\label{app:nash}
	\begin{proof}
The lemma follows by proving the following properties about the Nash equilibrium
strategies for the institution.
\begin{itemize}
    \item The Stackelberg threshold $\tau^*$ is a Nash equilibrium strategy.
\item All Nash equilibrium strategies lie in the interval $[\tau_0, \tau^*]$.
\item If $\tau_N$ is a Nash equilibrium strategy, then all $\tau \in [\tau_N,
    \tau^*]$ are also Nash equilibrium strategies.
\end{itemize}
Together, the three properties imply that the set of institution equilibrium
strategies is $[\tau_N, \tau^*]$ for some $\tau_N \in [\tau_0, \tau^*]$.

Before proceeding, we first establish and recall a few definitions. Let
$\Delta_{\tau}(x)$ be the best response of individual $x$ to the threshold $
\tau$. Define the set of individuals accepted for threshold $\tau$ and
response $\br_\tau$ by $\accs{\br_\tau}(\tau) = \{x : f_\tau(\br_\tau(x)) =
1\}$. Recall the \emph{outcome likelihood} $\ol(x) = \prob{Y = 1 \mid
X=x}$. Define the \emph{strategic outcome likelihood} as
$\old{\tau}(x') = \prob{Y = 1 \mid \br_\tau(X) = x'}$. The outcome likelihood is the probability that an individual is positive given their true features, while the strategic outcome likelihood is the probability that an individual is positive given their gamed features.

For the pair $(\tau, \br_\tau)$ to be a Nash equilibrium, $\tau$ must be a best
response to the individual's best response $\br_\tau(x)$. With knowledge of the
individual's response, $x' = \Delta_{\tau}(x)$, the institution's best response
is to play a threshold $\tau'$ so that $x' \in \accs{\br_\tau}(\tau')$ iff
$\old{\tau}(x') \geq 0.5$. Therefore, to show $(\tau, \br_\tau)$
is a Nash equilibrium, we must show $\tau = \tau'$, i.e. $x' \in
\accs{\br_\tau}(\tau)$ iff $\old{\tau}(x') \geq 0.5$.

To verify the condition $x' \in \accs{\br_\tau}(\tau)$ iff $\old{\tau}(x') \geq
0.5$,
there are three cases to consider. 
\begin{enumerate}
\item If $\ell(x) > \tau$, then $x \in \accs{\br_\tau}(\tau)$, $\br_\tau(x) =
x$, and $\old{\tau}(x) = \ol(x)$. Therefore, it suffices to check $\ell(x) \geq
0.5$.
\item If $\ell(x) < \tau$ and $\olcost(\ell(x), \tau) > 1$, then $x \not\in
\accs{\br_\tau}(\tau)$ and $\br_\tau(x) = x$. In this case, it suffices to check
$\ell(x) < 0.5$.
\item If $\ell(x) < \tau$ and $\olcost(\ell(x), \tau) \leq 1$, then $\br_\tau(x)
\in \accs{\br_\tau}(\tau)$, but $x \neq \br_\tau(x)$, so we must directly
verify $\prob{Y=1 \mid \olcost(\ell(X), \tau) \leq 1, \ell(X) \leq \tau} \geq
0.5$.
\end{enumerate}

We now proceed to the proof.

First, we show the Stackelberg equilibrium $(\tau^{*}, \br_{\tau^{*}})$ is a
Nash equilibrium. The Stackelberg threshold $\tau\opt$ is the largest $\tau\opt$
such that $c_L(0.5, \tau^*) \leq 1$.  If $\ell(x) > \tau\opt$, by monotonicity,
$\ell(x) \geq 0.5$. If $\ell(x) < \tau\opt$ and $c_L(\ell(x), \tau\opt) > 1$,
then $\ell(x) < 0.5$ by definition of $\tau\opt$. Similarly, if $\ell(x) <
\tau\opt$ and $c_L(\ell(x), \tau\opt) \leq 1$, then $\ell(x) \geq 0.5$, so
trivially  $\prob{Y=1 \mid \olcost(\ell(X), \tau) \leq 1, \ell(X) \leq \tau}
\geq 0.5$.
Hence, $(\tau\opt, \br_{\tau\opt})$ is a Nash equilibrium.
	
Next, we show that all Nash strategies must lie in the interval $[\tau_0,
\tau^*]$.
\begin{enumerate}
\item Suppose $\tau < \tau_0 = 0.5$. For all $x$ such that $\ell(\br_\tau(x)) =
\tau$, $\ell(x) < 0.5$. Therefore, $\prob{Y=1 \mid \olcost(\ell(X), \tau) \leq
1, \ell(X) \leq \tau} < 0.5$, so $\tau$ cannot be a Nash equilibrium strategy
for the institution.
\item Suppose $\tau > \tau^*$. By definition, $\tau\opt$ is the largest $\tau$
such that $c_L(0.5, \tau) \leq 1$. Thus, if $\tau > \tau\opt$, there exists $x$
with $\ell(x) < \tau\opt$ and $\olcost(\ell(x), \tau) > 1$, but $\ell(x) \geq
0.5$. Hence, $\tau$ cannot be a Nash strategy.
\end{enumerate}
		
Finally, we show that if $\tau_N$ is a Nash equilibrium strategy, then so is
$\tau$ for any $\tau \in [\tau_N, \tau^*]$. We consider each of the three cases
in turn.
\begin{enumerate}
\item Suppose $\ell(x) > \tau$. Then $\ell(x) > \tau > \tau_N \geq \tau_0 =
0.5$.
\item Suppose $\ell(x) < \tau$ and $\olcost(\ol(x), \tau) > 1$. Since $\tau \leq
\tau\opt$, by monotonicity in the second argument, $\olcost(0.5, \tau) \leq \olcost(0.5, \tau\opt) \leq 1$. By monotonicity in the first argument, since $\olcost(0.5, \tau)$, if $\olcost(\ol(x), \tau) > 1$, then it must be the case that $\ell(x) < 0.5$.
\item 
Suppose $\ell(x) < \tau$ and $\olcost(\ol(x), \tau) \leq 1$.
Let $L(\tau) = \{l : c_L(l, \tau) \leq 1, l \leq \tau \}$ be the set of outcome likelihoods that game to the threshold $\tau$ and $l_\tau = \min_{l \in L_\tau} l$ be the minimum such outcome likelihood. The points that game under the thresholds $\tau_N$ and $\tau$ form the intervals $[l_{\tau_N}, \tau_N]$ and $[l_\tau, \tau]$, respectively. Since $l_{\tau_N} \leq l_\tau$ and $\tau_N \leq \tau$, we have that
\begin{align*}
& \prob{Y = 1 \mid c_L(\ol(X), \tau) \leq 1, \ol(X) \leq \tau} \\
& = \prob{Y = 1 \mid \ol(X) \in [l_{\tau}, \tau]} \\
& \geq \prob{Y = 1 \mid \ol(X) \in [l_{\tau_N}, \tau_N]} \\
& = \prob{Y = 1 \mid c_L(\ol(X), \tau_N) \leq 1,
\ol(X) \leq \tau_N} \geq 0.5
\end{align*}
where the last inequality holds because $t_N$ is a Nash strategy.
\end{enumerate}
Since each of the three cases are satisfied, any $\tau \in [\tau_N, \tau\opt]$
is a Nash strategy.

We have now demonstrated each of the three properties outlined at the
beginning, and the lemma follows.
\end{proof}
\end{appendices}

\end{document}